\documentclass{easychair}
\pdfoutput=1
\usepackage{algorithm2e}
\usepackage{amsthm} 
\usepackage{amsfonts}

\hypersetup{pdfborder=0}

\newcommand{\csp}{\textsf{CSP}}
\newcommand{\CSP}{\emph{CSP}}
\newtheorem*{mydef}{Definition}
\newtheorem{lemma}{Lemma}
\begin{document}

%% Front Matter
%%
% Regular title as in the article class.
%
\title{The Exact Closest String Problem as a \\Constraint Satisfaction Problem}

% \titlerunning{} has to be set to either the main title or its shorter
% version for the running heads. Use {\sf} for highliting your system
% name, application, or a tool.
%
\titlerunning{The {\csp} as a {\CSP}}

% For only the editors. Authors, please keep this commented out
%\volumeinfo
%	{G. Sutcliffe, A. Voronkov} % editors
%	{2}                         % number of editors
%	{{\easychair} 1.0, 2008}    % event
%	{1}                         % volume
%	{1}                         % issue
%	{1}                         % starting page number

% Authors are joined by \and and their affiliations are on the
% subsequent lines separated by \\ just like the article class
% allows.
%
\author{Tom Kelsey\thanks{To whom all correspondence should be addressed.}\\
University of St Andrews\\
St Andrews, KY16 9SX United Kingdom\\
\url{tom@cs.st-andrews.ac.uk}\\
\and
Lars Kotthof\/f\\
University of St Andrews\\
St Andrews, KY16 9SX United Kingdom\\
\url{larsko@cs.st-andrews.ac.uk}\\
 }

% \authorrunning{} has to be set for the shorter version of the authors' names;
% otherwise a warning will be rendered in the running heads.
%
\authorrunning{Kelsey, Kotthof\/f}

\maketitle

%------------------------------------------------------------------------------
% Abstract
%
\begin{abstract}
 We report (to our knowledge) the first evaluation of Constraint Satisfaction as a computational framework for solving closest string problems. We show that careful consideration of symbol occurrences can provide search heuristics that provide several orders of magnitude speedup at and above the optimal distance.  
We also report (to our knowledge) the first analysis and evaluation -- using any technique -- of the computational difficulties involved in the identification of all closest strings for a given input set.  
 We describe algorithms for web-scale distributed solution of closest string problems, both purely based on AI backtrack search and also hybrid numeric-AI methods.
\end{abstract}

%------------------------------------------------------------------------------
\section{Introduction}
\label{sect:introduction}

The closest string problem ({\csp}) takes as input a set of strings of equal length over a fixed alphabet. A solution is a string with the smallest possible maximum Hamming distance from any input string. (Strictly speaking, distance with respect to any suitable metric can be minimised; the Hamming distance is the standard edit distance metric used for this class of problems.) {\csp} has applications in coding and information theories  (in these fields the problem is also known as minimum radius), but when the input strings consist of nucleotide sequences over the letters A, C, G and T,  (or of  mRNA sequences over the letters A, C, G and U, or of  amino acid sequences over an alphabet of size 20) the {\csp} has important applications in computational biology (where the problem class is also known as centre string). 
Examples include the identification of consensus patterns in a set of unaligned DNA sequences known to bind a common protein \cite{Hertz1990}, finding conserved secondary structure motifs in unaligned RNA sequences \cite {Pavesi2004}, discovering motifs in ranked lists of DNA sequences \cite{Eden2007}, finding DNA regulatory motifs within unaligned noncoding sequences \cite{Roth1998}, the identification of sister chromatids by DNA template strand sequences \cite{Falconer2010}, and DNA motif representation with nucleotide dependency \cite{Chin2008}.
Our aim is to provide theoretical and practical results -- together with empirical supporting evidence  -- that lead to improved  {\csp} solution for biological problems, so in this paper the base alphabet $\Sigma$ will always consist of four symbols. 

A Constraint Satisfaction Problems ({\CSP})  consists of a set of constraints involving variables taking discrete values. A solution to a {\CSP} is an assignment of values to variables such that no constraint is violated. {\CSP} solvers are used for many important classes of problems for which solutions must take discrete values, but,
to our knowledge, the closest string problem has not been modelled and solved as a {\CSP}. The research question under consideration, therefore, is ``Is   {\CSP}  a useful framework for solving {\csp} instances?''

In this paper we investigate approaches to developing and solving such models. We demonstrate that a careful choice of search heuristic can give several orders of magnitude speedup in general. We show that {\CSP} modelling and solution are effective tools for the related problem of obtaining {\bf all} closest strings. We consider the distribution of closest string problems across a cloud (or grid, or cluster) of computing nodes, and identify two potential super-linear speedups that can be achieved in practice. Finally we identify the strengths and weaknesses of existing numeric approaches, and suggest hybrid discrete and numeric methods that combine the best features of {\CSP} search and numeric search for solutions. 

In the rest of this introduction we discuss existing methods for the {\csp}  with respect to theoretical complexity results,  give brief overviews of Constraint Satisfaction theory and the Minion {\CSP} solver, and formally define the theoretical concepts upon which the research is based.  In Section \ref{CSPasCSP} we model closest string problems as {\CSP}s,  compare search heuristics, and provide results for the all closest string problem. We describe distributed algorithms in Section \ref{cloud}, both for pure {\CSP} models and heuristics, and for hybrid {\CSP}-numeric methods. In Section \ref{discussion} we discuss the relative strengths and limitations of {\CSP} as a framework for closest string identification, and identify future avenues of research.

%%------------------------------------------------------------------------------
\subsection{Computational Complexity and Existing Methods}
\label{complexity}

{\csp} has been shown to be NP-complete for binary strings \cite{Frances1997}
and for alphabets of arbitrary size \cite{Lanctot2004}.  Intuitively there are
$|\Sigma|$ choices for each of the $L$ positions in any candidate closest
string where $\Sigma$ is the alphabet, and for any algorithm that fails to check
each of this exponential number of cases one could devise a {\csp} for which the
algorithm returns an incorrect result. 
 
 Approximate solutions to within $(4/3 + \epsilon)$ of the minimal $d$ can be obtained in polynomial time \cite{Lanctot2004,Li2002}, with several practically useful implementations available, notably those based on genetic algorithms \cite{Julstrom2009}.  However, in this paper we are concerned with first finding exact solutions, and then (given that we know the minimal distance $d$)  finding all closest strings that are within $d$ of $S$. Clearly, an approximate method will not, in general, identify the minimal $d$, and therefore can not be used as a basis for finding all solutions. 
 
 Excellent exact results -- provided that close bounds have already been identified -- have been obtained by modelling the {\csp} as an Integer Programming Problem \cite{Meneses2004}, and solving the resulting instances using numerical branch and bound methods \cite{Land1960}.  This form of search differs from the backtrack search used by {\CSP} solvers by having a much less organised search pattern. This is often advantageous, but can be a hindrance when searching for all solutions: IP branch and bound is optimised for optimisation, as it were, rather than exhaustive search for all candidates for a constant objective function. If the IP formulation suggested in \cite{Meneses2004} is used, then  the feasible region deliberately excludes optimal solutions in order to reduce the numbers of variables, in which case no search for all solutions can be made. 
 
 A linear time algorithm exists for solutions to the {\csp} for fixed distance $d$ \cite{Gramm2001}. The exponential complexity is now in the coefficient, as the method is $O(NL + Ndd^d)$ where the problem has $N$ strings of length $L$.  

%%%------------------------------------------------------------------------------
%\subsection{Applications in Computational Biology}
%\label{compbiol}

%%------------------------------------------------------------------------------
\subsection{Constraint Satisfaction Problems}
\label{csp}

 \begin{mydef}
 \label{def:csp}
A Constraint Satisfaction Problem  $\Upsilon$ is a set of constraints $\mathcal{C}$ acting on a finite set of
variables $\Delta:= \{ A_1, A_2, \ldots, A_n\}$, each of which has a finite 
domain of possible values $D_i:= D(A_i) \subseteq \Lambda$. A
\emph{solution} to $\Upsilon$ is an instantiation of all of the variables
in $\Delta$ such that no constraint in $\mathcal{C}$ is violated.
\end{mydef}

The class of {\CSP}s is NP-complete as it is a generalisation of propositional satisfiability (SAT).  The Handbook of  Constraint Programming~\cite{CSP} provides full details of {\CSP} theory and techniques.  A key observation is that different models (i.e. choices of variables, values and constraints) for the same problem (or class of problems) will often give markedly different results when the instances are solved, but, as with numeric Linear, Mixed-Integer and Quadratic Programming, there is no general way to decide in advance which candidate models and heuristics will lead to faster search. 

A typical solver  operates by building a search tree in which the nodes are assignments of values to variables, and the edges lead to assignment choices for the next variable. If  a constraint is violated at any node, then search backtracks. If a leaf is reached, then all constraints are satisfied, and the full set of assignments  provides a solution. These search trees are obviously exponential, and in the worst-case scenario every node may have to constructed. However, large-scale pruning of the search tree can occur by judicious use of consistency methods. The idea is to do a small amount of extra work that (hopefully) identifies variable-value assignments that are already logically ruled out by the current choice of assignment, meaning that those branches of the search tree need not be explored. While there are no guarantees that this extra work is anything other than an overhead, in practice enough search is pruned to give efficient solutions for otherwise intractable problems. Taking a specific  example from the empirical evaluation reported later in this paper, an all closest string problem for a fixed distance involving strings of length 25 with a 4-symbol alphabet will require at most $4^{25}  \approx 1.1 \times 10^{15}$ nodes to be searched. An efficient solver will search only 3 or 4 $\times 10^9$ nodes, with the remainder being ruled out by efficient propagation of the logical results of the assignments during search. Moreover, an efficient solver will search around 300,000 of the remaining nodes per CPU second. It is this efficient reduction in search space that allows {\CSP} practitioners to solve otherwise intractable problems.

Heuristics exist for choices of variable-value pair for the next node, and as before these may have no effect on the number of nodes visited.  Again, in general, variable and value orderings designed for specific problem classes can lead to several orders of magnitude reduction in the number of nodes needed to find a solution. Standard choices for variable orderings include random, smallest domain, largest domain, most-constrained (i.e. chooses avariable that appears in a maximal number of constraints), least-constrained, etc.  Results will vary with the problem class and model under consideration. Taking another specific example from experiments in this paper involving closest string problems, enforcing singleton singleton arc consistency -- a limited depth procedure that aims to prune entire branches near the root, see \cite{Bessiere2008} for a full analysis -- at the root node of a search tree  can be a huge loss. It can take three times longer to reach the first closest string at a given distance than it takes to find  all closest strings.

In summary, the solution performance for instances of a class of  {\CSP}s will depend crucially on choices of model, consistency, search order and the solver used.  Moreover, empirical evaluation is often the only way to decide which of these choices is better for a given set of circumstances.

%%------------------------------------------------------------------------------
\subsubsection{The Minion {\CSP} solver}
\label{minion}

 The constraint solver Minion~\cite{Minion}  uses the memory architecture of modern computers to speed up the backtrack process compared to other solvers.  Minion has an extensive set of constraints, together with efficient propagators that enforce consistency levels very rapidly.  Minion has been used to solve open problems in combinatoric algebra \cite{Distler2009}, finding billions of solutions in a search space of size $10^{100}$ in a matter of hours.   Minion is used as the solver for this investigation as it offers both fast and scalable constraint solving,  which are important factors when solving closest string problems. Moreover, the user can easily specify bespoke variable orderings, and less easily specify value orderings. 

%%------------------------------------------------------------------------------
\subsection{Formal Definitions and Results}
\label{concepts}

Before proceeding to the technical Sections, we first formalise Hamming Distances and Diameters, and closest strings:

\begin{mydef}
Let $S_1$ and $S_2$ be strings of length $L$ over an alphabet $\Sigma$. Let $D$ be the binary string of length $L$ such that
$$
D(i)= \left\{
\begin{array}{c l}
  1 &  S_1(i) \neq S_2(i)\\
  0 & \mbox{otherwise}
\end{array}
\right.
$$
The \emph{Hamming Distance}  $hd(S_1,S_2)$ is defined as the sum from $i=1$ to $L$ of the $ D(i)$. 
\end{mydef}

\begin{mydef}
Let $S= \{S_1,S_2,\ldots,S_N\}$ be a set of strings of length $L$ over an alphabet $\Sigma$.  
A \emph{Closest String}  to $S$ is defined as any string $CS$ of length $L$ over  $\Sigma$ such that
$$
hd(CS,S_i)  \leq d \quad \forall i \in \{1,2,\ldots,N\}
$$
with $d$ being the minimal such distance for $S$. The \emph{Hamming diameter} $HD$ of $S$ is defined as
$$
HD(S) = \max(hd(S_i,S_j)) \quad \forall i,j \in \{1,2,\ldots,N\}.
$$
\end{mydef}

A solution to a closest string problem involving the strings in $S$ is therefore a string $CS$ and a minimal distance $d$ such that each member of $S$ is within $d$ of $CS$. The Hamming distance is an edit distance that quantifies the number of substitutions from $\Sigma$ required to turn one string into another. It is easy to show that Hamming distance is a metric, satisfying the triangle inequality.  It is clear that the Hamming Diameter is an upper bound for the distance of a closest string: a candidate closest string at a greater distance can be replaced by any member of $S$, reducing the maximal distance to $HD(S)$.  We can obtain a lower bound for the distance of a closest string by observing that the distance can not be less than half the Hamming Diameter:

\begin{lemma}
\label{lemma1}
Let $S= \{S_1,S_2,\ldots,S_N\}$ be a set of strings of length $L$ over an alphabet $\Sigma$. A closest string $CS$ to $S$ must be within $\lceil HD(S)/2 \rceil$ of $S$.
\end{lemma}

\begin{proof}
Let $S_i$ and $S_j$ be two strings from $S$ for which the Hamming Diameter is achieved, and let $S_k$ be any other string of length $L$ over $\Sigma$. By the triangle inequality $H(S) = hd(S_i,S_j) \leq hd(S_i,S_k) + hd(S_j,S_k)$.   If (without loss of generality)  $hd(S_i,S_k) < \lceil HD(S)/2 \rceil$ then $hd(S_j,S_k) \geq \lceil HD(S)/2 \rceil$.  Hence any distance less than $\lceil HD(S)/2 \rceil$ can not be a maximal distance from $S_k$ to any string in $S$.
 \end{proof}
 
 Search space reduction can be achieved by noting that any value not appearing in position $j$ of any of the strings in $S$ need not appear in a closest string solution.  It should be noted that this only applies when searching for the first optimal solution. When searching for all solutions, any symbol from $\Sigma$ can, in principle, appear at any position in $CS$. 
 
\begin{lemma}
\label{lemma2}
Let $S= \{S_1,S_2,\ldots,S_N\}$ be defined as in Lemma \ref{lemma1}. Let $\Sigma_j$ for $j \in 1,2,\ldots,L$ denote the subset of $\Sigma$ obtained by selecting every symbol that appears in position $j$ of a string in $S$.  Then any symbol in position $j$ of  a closest string to $S$ must also be in $\Sigma_j$.
\end{lemma}

\begin{proof}
Suppose symbol $s$ in $\Sigma \setminus \Sigma_j$ appears in position $j$ of a solution $CS$. Let $CS^*$ be the string consisting of $CS$ with $s$ replaced by a symbol from $\Sigma_j$ at position $j$. Then $CS^*$ is strictly closer to those strings in $S$ with that symbol at that position, and distance to all other strings is unchanged. Hence if the current $d$ is optimal for $CS$, it remains optimal for $CS^*$.
 \end{proof} 
 
 The final definition needed for this investigation encapsulates  frequencies of symbol appearances per string position, and will be used in Section \ref{pwm} to direct backtrack search for closest strings.

\begin{mydef}
Let $S= \{S_1,S_2,\ldots,S_N\}$ be a set of strings of length $L$ over an alphabet $\Sigma$.  
A \emph{ Position Weight Matrix (PWM)}  for $S$  is an $|\Sigma| \times L$ matrix with entries
$PMS_S(i,j)$ defined as the frequency of symbol $i$ appearing at position $j$  in $S$.
\end{mydef} 

\noindent An example Position Weight Matrix is given in Figure \ref{PWMfig}.
 
\begin{figure}[htb!]
	\begin{centering}
	\scalebox{0.6}{\includegraphics{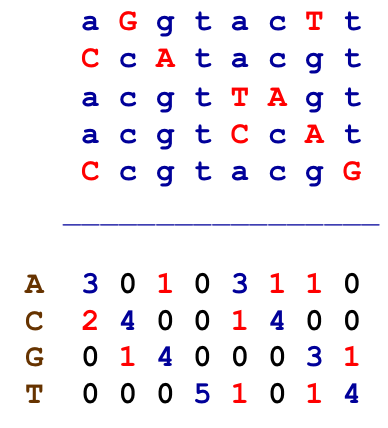}}
	\caption{Five strings of length 8 are shown above, with their PWM shown below.}
	\label{PWMfig}
	\end{centering}
\end{figure}

%\begin{table}
%\begin{center}
%\begin{tabular}{ccccccccc}
% & a & \textbf{G} & g & t & a & c & \textbf{T} & t\\
% & \textbf{C} & c & \textbf{A} & t & a & c & g & t\\
% & a & c & g & t & \textbf{T} & \textbf{A} & g & t\\
% & a & c & g & t & \textbf{C} & c & \textbf{A} & t\\
% & \textbf{C} & c & g & t & a & c & g & \textbf{G}\\
%\hline
%A & 3 & 0 & 1 & 0 & 3 & 1 & 1 & 0\\
%C & 2 & 4 & 0 & 0 & 1 & 4 & 0 & 0\\
%G & 0 & 1 & 4 & 0 & 0 & 0 & 3 & 1\\
%T & 0 & 0 & 0 & 5 & 1 & 0 & 1 & 4
%\end{tabular}
%\caption{Five strings of length 8 are shown above, with their PWM shown below.}
%\label{PWMfig}
%\end{center}
%\end{table}

%%%%%%%%%%%%%%%%%%%%%%%%%%%%%%%%%%%%%%%%%%%------------------------------------------------------------------------------
\section{Closest String as a Constraint Satisfaction Problem}
\label{CSPasCSP}

Using the terminology from Sections \ref{csp} and \ref{concepts}, we now construct a {\CSP} instance from
a given closest string problem.  For the purposes of this paper, the alphabet $\Sigma$ consists of the numbers 1, 2 ,3 and 4, representing A, C, G and T respectively. Clearly this artificial restriction can easily be relaxed in order to model arbitrary alphabets. 

Given $S$, a set of $N$ strings of length $L$ over alphabet $\Sigma$, we first compute the Hamming Diameter $HD(S)$ and use this to provide a lower bound, $d_{min}$, for the optimal distance $d$, as shown in Lemma \ref{lemma1}.
 $\Upsilon(S,d_{min},HD(S))$ denotes the {\CSP} instance in which  the set of variables is $\Delta:= \Delta_1 \cup \Delta_2 \cup \Delta_3 \cup \Delta_4$, where
 \begin{enumerate}
\item $\Delta_1$ is the array $[CS_1,CS_2,\ldots,CS_L]$ of variables representing the closest string, each such variable having domain 1 through 4
\item $\Delta_2$ is an $N \times L$ array of binary variables used to calculate Hamming Distances from $\Delta_1$ to the input strings $S$
\item $\Delta_3$ is the array  $[D_1,D_2,\ldots,D_N]$ of variables representing the distance of each string in $S$ to the current $CS$ candidate, each such variable having domain $d_{min}$ through  $HD(S)$
\item  $\Delta_4$  is the single distance variable $d$ with domain  $d_{min}$ through  $HD(S)$.
\end{enumerate}
 
\noindent The constraints are:
\begin{enumerate}
\item  $\Delta_2(i,j) = 0$ iff $S_i(j)= \Delta_1(j)$ 
\item   $\Delta_3(k)$ is the sum of row $k$ of $\Delta_2$
\item   $\Delta_4$ is the maximum value appearing in $\Delta_3$
\item   $\Delta_4$ is minimised: if a solution is found with $\Delta_4 = d$, search for another solution with $\Delta_4 = d-1$ (unless $d = d_{min}$).
\end{enumerate}

\noindent $\Delta_1$ are the search variables: nodes of the search tree consist of values assigned to these variables. $\Delta_4$ is the objective function (or cost function).  A returned solution is $\Delta_1 \cup \Delta_4$, a closest string together with the optimal distance.  Solving $\Upsilon(S,d_{min},HD(S))$ is guaranteed to return a solution, although it is not impossible that all $4^L$ nodes are visited for every current minimal $d$. Restricting the domains of  $\Delta_3$ will save computational effort when a solution is found with $d = d_{min}$  and will have no effect otherwise. Restricting the domains of the $\Delta_1$ variables in line with lemma \ref{lemma2} also reduces the search space, although the restrictions can not apply when searching for all solutions. 

To find all closest strings $\Upsilon(S,d_{min},HD(S))$ is solved to obtain $CS$ and $d_{opt}$. By restricting the domains of $\Delta_3$ to   $d_{min}$ through $d_{opt}$ and removing the optimisation constraint we obtain a new {\CSP}  $\Upsilon^*(S,d_{min},d)$ which can be solved for all solutions.  The search undertaken to find the first solution $CS$ need not be repeated: constraints can be added that rule out those parts of the search tree already processed. It should also be noted that $CS$ and $d$ need not be obtained using the Constraint Satisfaction approach: any method that returns an optimal solution can be used to create an all closest strings  {\CSP}.

 % 
%%------------------------------------------------------------------------------
\subsection{Position Weight Matrix Variable and Value Ordering}
\label{pwm}

We now use results from computational biology to devise a bespoke variable and
value ordering schema for $\Upsilon(S,d_{min})$. By precalculating a Position
Weight Matrix for $S$ as defined in Section \ref{concepts} we can order the
search variables by maximum frequency. For each variable, we order the values
assigned during search by decreasing relative frequency.  Tie breaks are either
random or by least index.  In the example given as Figure \ref{PWMfig} the
variable ordering by position 1 through 8 would be 1: position 4 (having 5
occurrences), 2 -- 5: positions 2, 3, 6, and 8 in any order (each having 4
occurrences), 6--8: positions 1 and 7 in any order (having the least highest
frequency of 3).  The value ordering for position 5 in the figure would be 1: A
(most frequent), 2--3: T and C in any order, 4: G (least frequent).  By Lemma \ref{lemma2}, when seeking a single solution it is safe to exclude values
that don't appear at a given position from their respective variable domains
before search. Hence in for the example in Figure \ref{PWMfig} the value ordering would be values typeset in blue followed by values typeset in red, with black values excluded. 

The idea behind this search heuristic is that search starts close to (in the sense of maximum likelihood) an optimal solution. Only if no such solution is found does search progress to less likely (but not impossible)  parts of the search tree .

%%------------------------------------------------------------------------------
\subsection{Comparison of Search Heuristics}
\label{compare}

\begin{figure}
\begin{center}
\includegraphics[width=.8\textwidth]{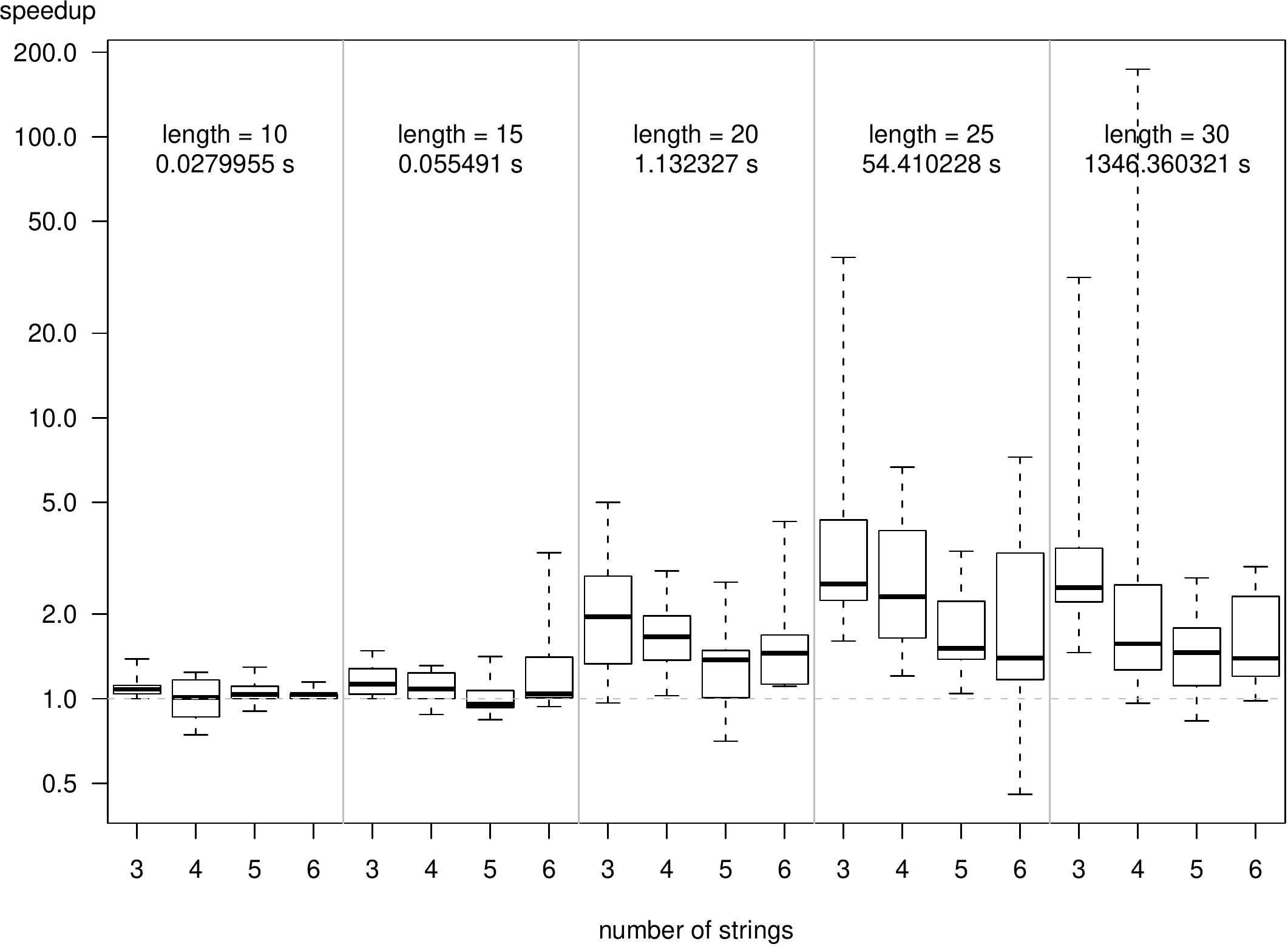}
\includegraphics[width=.8\textwidth]{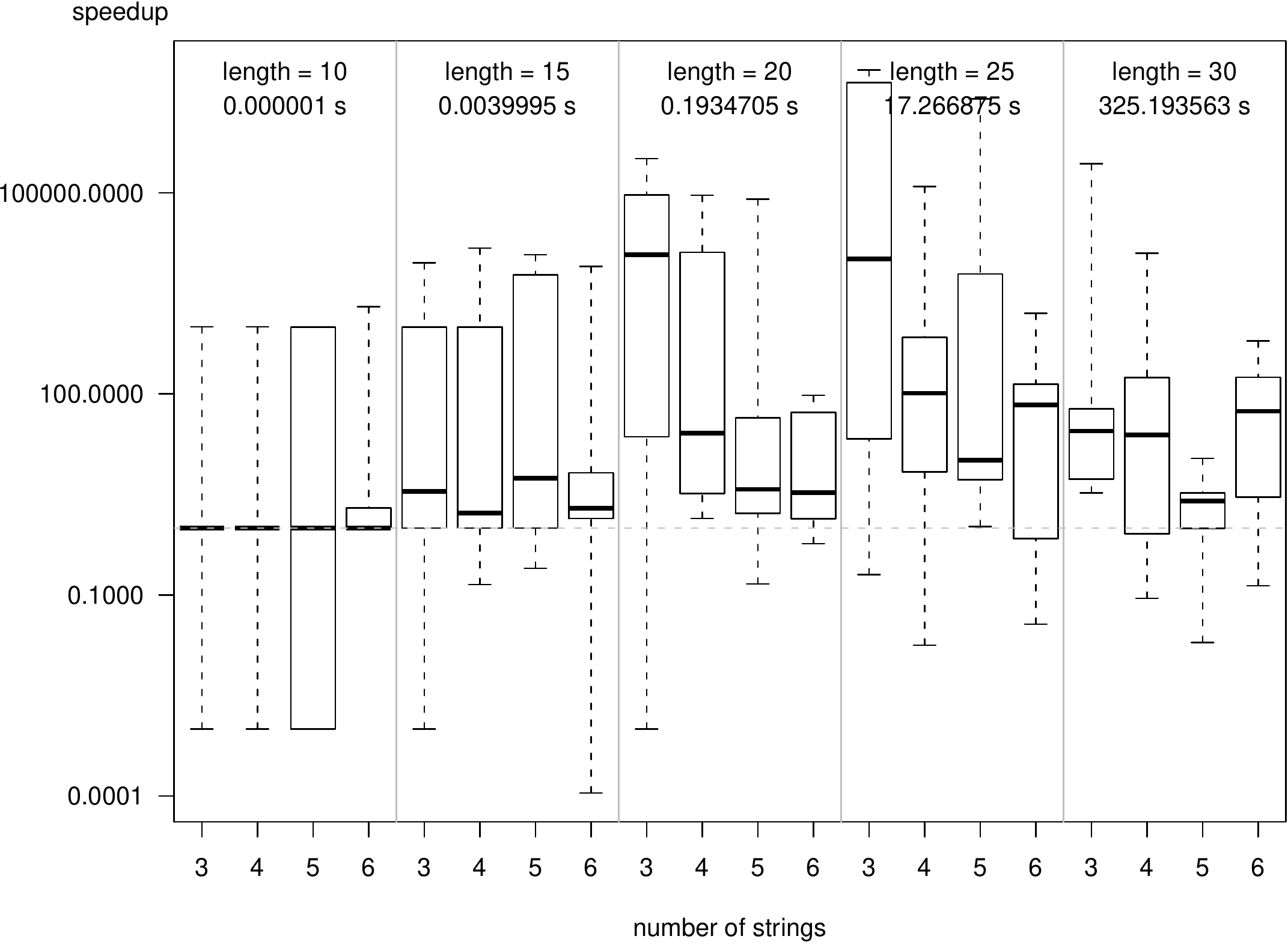}
\caption{Empirical data from 200 instances of 3, 4, 5 and 6 strings of lengths
10, 15, 20, 25 and 30. For each combination of parameters, 10 random instances
were generated with results summarised in the boxes which show median values (thick line), 25th--75th percentiles (boxed) and 0th--100th percentiles (dashed lines).  In the top panel we compare the exact optimal solution times. In the lower panel we show the times taken to obtain an optimal result, omitting the time needed to certificate that result. 
In both figures the y axis shows the speedup
of Position Weight Matrix over Smallest Domain First ordering on a logarithmic scale, and the times given below the string lengths are the median CPU time taken over all strings of that length.   The experiments
were conducted on a dual quad-core 2.66 GHz Intel X-5430 processor with 16 GB of
RAM.}
\label{searchres}
\end{center}
\end{figure}

In this Section we test the hypothesis that PMS-based search heuristics reduce
the search needed for solutions to $\Upsilon(S,d_{min})$ {\CSP} instances when
compared to a standard heuristic. Figure \ref{searchres} illustrates the results
from 200 closest string problems.  Each problem was run first with smallest
domain variable ordering and ascending value ordering (Minion defaults), and
then with PWM-based variable and value ordering as described in Section
\ref{pwm}. 

For exact solutions -- upper panel of Figure \ref{searchres} -- we observe an improvement of PWM over SDF in almost all cases. The
speedup is as high as several orders of magnitude in some cases. The difference
is statistically significant at the 0.001 level. These results
are as expected: the PWM reflects the maximum likelihood of a closest string, so
a search that respects these likelihoods will nearly always be highly efficient,
but will visit very many non-essential nodes  on the few occasions that the
maximum likelihood does not lead to a closest string. A key observation is that the magnitude of speedup increases with increasing string length, which is highly encouraging since the complexity of closest sting is exponential in string length.

If we only consider the search effort needed to find an optimal solution (not taking into account the work needed to provide a certificate of optimality by ruling out closer strings at lower distance) then the speedup of PWM over smallest domain is at the level of orders of magnitude in the general case -- Figure \ref{searchres}, lower panel. This indicates that heuristics are less important when searching exhaustively at a lower than optimal distance: most of the practical complexity of closest string search is associated with providing certificates of optimality, rather than identifying close strings which turn out to be optimal.

\begin{figure}[htb!]
\begin{center}
\includegraphics[width=.6\textwidth]{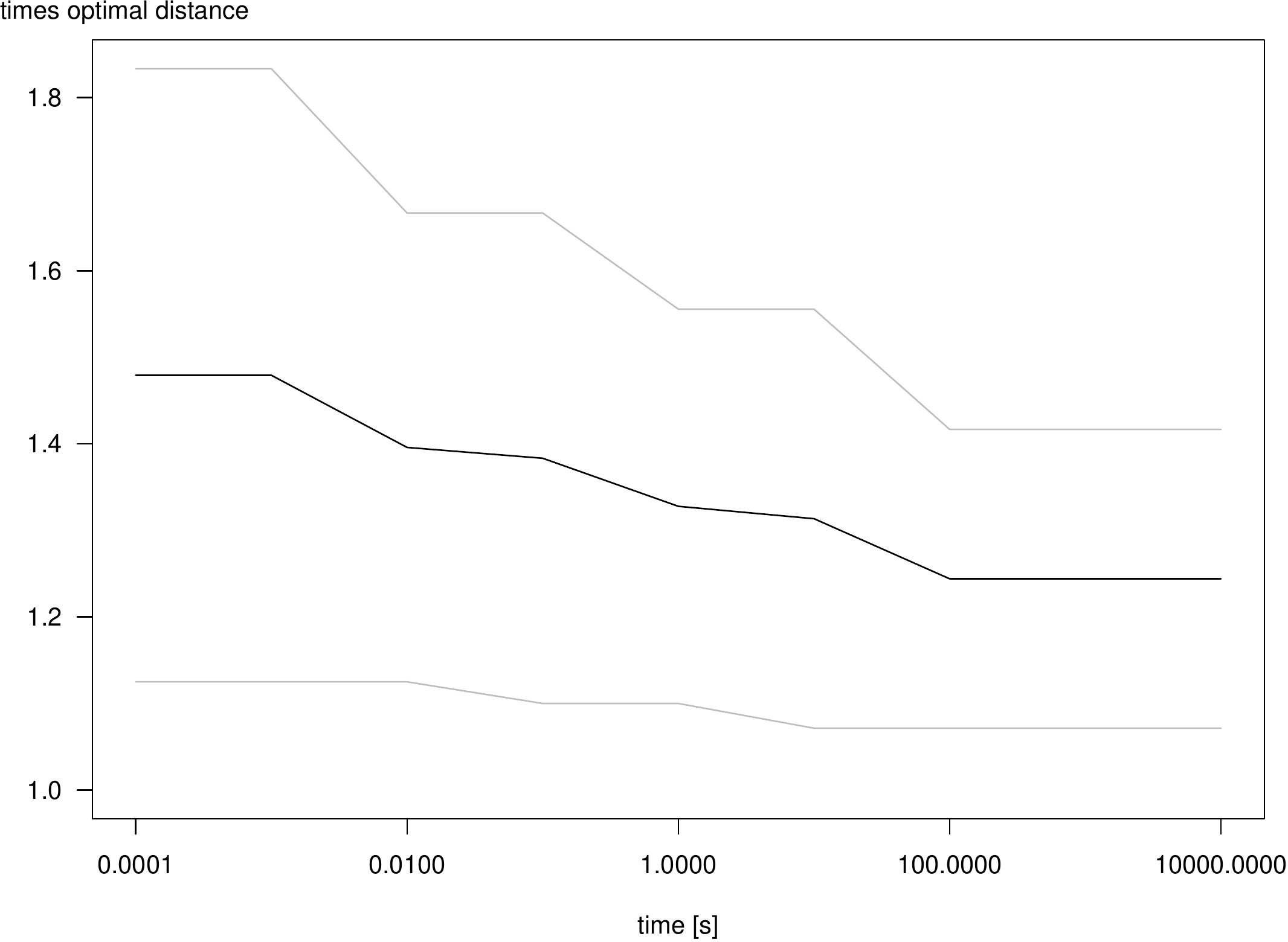}
\caption{Convergence towards optimal Hamming distance. The upper line is the
maximum relative distance, the lower line the minimum and the middle line the mean of the 200 experiments performed. The y axis denotes multiples of the optimal Hamming distance, the x axis denotes CPU time for the PWM heuristic on a logarithmic scale. The experiments
were conducted on a dual quad-core 2.66 GHz Intel X-5430 processor with 8 GB of
RAM.}
\label{convergence}
\end{center}
\end{figure}

Figure~\ref{convergence} shows that we  achieve a good approximation very quickly, in line with existing results that guarantee approximation to four thirds of optimality in polynomial time  \cite{Lanctot2004,Li2002}.  This motivates the hybrid symbolic-numeric methods detailed in Section \ref{hybrid}: practitioners can use {\CSP} to obtain good bounds quickly, then use either numeric methods or AI search methods -- or indeed both using a distributed architecture -- to explore the remaining search space for an exact solution plus certificate of optimality.

Taken together the results indicate that:
\begin{enumerate}
\item  {\CSP} search with PWM variable-value ordering will (in general)  efficiently find candidate solutions to closest string problems with decreasing maximum distance $d$
\item  {\CSP} search with any sensible search ordering can be used to exhaustively rule out the distance below the optimal $d$
\item our empirical evidence is in line with previously reported results: an approximate solution to closest string can be computed in polynomial time, but computation of the necessary certificates of optimality remains intractable in the general case
\item  sequential, single-processor {\CSP} search for problems having more strings of greater length (and possibly a larger alphabet) will become intractable  due to the inherent NP-completeness of closest string.
\end{enumerate}

%%------------------------------------------------------------------------------
\begin{figure}[htb!]
	\begin{centering}
	\scalebox{0.6}{\includegraphics{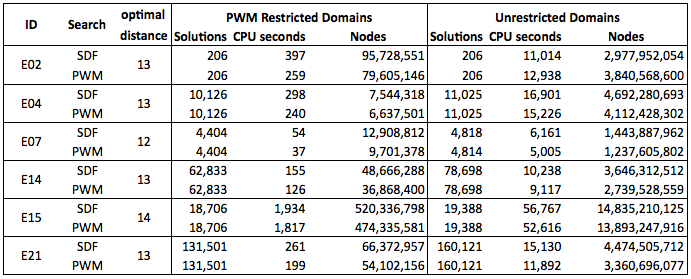}}
	\caption{Empirical data from 6 instances of 5 randomly generated strings of
    length 25. SDF and PWM indicate smallest domain and position weight heuristics respectively. All timings were calculated using a dual quad-core 2.66 GHz Intel X-5430 processor with 16 GB of RAM.}
	\label{searchall}
	\end{centering}
\end{figure}

\subsection{All Closest Strings}
\label{all}

To our knowledge, no study has investigated the problem of finding all closest strings for a given set $S$. This may be due to the additional computational complexity involved: it is hard enough to find single exact closest strings without performing a systematic search for all such strings having the  same maximum distance from $S$. It may also be the case that the problem is not interesting: the important information in a {\csp} solution being the distance returned, with the closest string being merely an exemplar at that distance. It seems likely, however, that knowledge of how much self similarity an input set has -- rather than just the degree of self similarity -- could be useful information  in sequence analysis.

Despite this uncertainty, we wish to investigate the effect that modelling has on the set of all solutions. In our {\CSP} model described in Section \ref{CSPasCSP} we reduce the search space for a first solution by forbidding any variable to take a value that is not present at that position in one of the input strings. Similar restrictions were made by Meneses et al. when formulating {\csp} as an Integer Programming problem \cite{Meneses2004}. The questions are:

\begin{enumerate}
\item Are many otherwise closest strings ruled out by these restrictions?
\item How much extra computational effort is required to identify each and every closest string?
\end{enumerate}

In Figure \ref{searchall} we show the results of sample calculations for six instances of 5 randomly generated strings of length 25 . We see that in all cases (columns headed PWM Restricted Domains)  it is relatively easy using Minion to identify all closest strings if we restrict search to those alphabet symbols that have non-zero values in the position weight matrix for the instance. We also find that search using PMW ordering heuristics performs marginally better than straightforward smallest domain heuristics. 

When the variable-value assignments that have been ruled out because the value does not appear at the variable's position in any element of $S$ are added back to the domains of the search variables, we can perform full search for all closest strings (Figure \ref{searchall}, columns headed by Unrestricted Domains).  The percentage of new closest strings found ranges from 0\% to 22\%, but increase in search required is typically two orders of magnitude.  It should be noted however that:

\begin{enumerate}
\item Minion is searching far fewer than the $4^{25}$ possible search nodes for each instance, the majority being pruned by efficient propagation of the logical consequences of the variable-value assignments implicit  at each node, and
\item Minion is searching 250,000 -- 300,000 nodes per second in addition to the work involved in identifying search sub-trees that need not be explored.
\end{enumerate}

%%%%%%%%%%%%%%%%%%%%%%%%%%%%%%%%%%%%%%%%%%%%%%%%%%------------------------------------------------------------------------------
\section{Distributed and Hybrid Computing Strategies}
\label{cloud}

\IncMargin{1em}
\begin{algorithm}[ht]
  \SetKwBlock{Parallel}{do in parallel}{end}
      \SetKwFunction{Solved}{Solved?}
  \SetKwInOut{Input}{Input}\SetKwInOut{Output}{Output}

  \Input{A {\CSP} $\Upsilon$, a cutoff period $T_{max}$ and a branching factor  $K$}
  \Output{Either the first solution, or a guarantee that there are no solutions}
  \BlankLine
  \While {not \Solved{$\Upsilon$}}{
          Send $\Upsilon$ to a node\\
          Run solver with input $\Upsilon$ for $0 \leq t \leq T_{max}$\\
  \lIf{ \Solved{$\Upsilon$}}{\\
  Return solution\\
   }%   
 \Else{ % else
 $\Upsilon \leftarrow \Upsilon$ with new constraints ruling out search already performed\\
 Split $\Upsilon$ into $K$ subproblems $\Upsilon_1, \Upsilon_2, \ldots, \Upsilon_K$\\
 \Parallel{%
     \For{$1 \leq k \leq K$}{
         Solve($\Upsilon_k$, $T_{max}$,$K$)
        }
  }
}%     
 }
  
\caption{A recursive distributed algorithm to solve  any {\CSP} }\label{alg:cloud}
\end{algorithm}\DecMargin{1em}

%%%%----------------------------------------------------
\subsection{Distributed {\CSP}}

Given the inherently exponential increase in search effort involved in providing
a certificate for an optimal closest string distance by ruling out any closest
strings with with lower distance, the exact solution of large-scale problems is
not expected to be tractable using purely sequential search. In this Section we
describe algorithms that distribute search across multiple compute nodes. These
algorithms will solve closest string problems either on a cluster (a local group
of homogeneous nodes), a grid (a more loosely coupled, heterogeneous and
geographically dispersed set of nodes), or a cloud (a set of an unknown number
of nodes in unknown locations, each having  unknown architecture and resource).
Generally speaking, a cluster is more controllable but smaller than a cloud,
with a grid being either the best or worst of both worlds, depending on one's
point of view. For our purposes we do not require any communication across nodes
(although computational efficiencies could be obtained if that were the case),
and can therefore treat the three distributed paradigms as a single approach.
The only disadvantage to using a cloud is that empirical evaluation is often
impossible since the times reported in virtual machines are not reliable. This
is because clocks of virtual machines can be slowed down or sped up by the VM
management software.  We therefore prototype our computational methodology on a
cluster, and, when satisfied that it is efficient, deploy using a cloud to take
advantage of the very large number of nodes available.  

Algorithm \ref{alg:cloud} gives the basic structure of our distributed search. The predicate {\emph Solved?} returns true whenever the input {\CSP} finds the first solution or finishes searching the entire tree without finding a solution. It returns false if either the computation has timed out, or the node has suddenly stopped working for some reason. 
If all solutions to the input {\CSP} are required, then we modify Algorithm
\ref{alg:cloud} so that all solutions found so far are returned whenever the
{\emph Solved?} predicate fails.  

It should be stressed that Algorithm \ref{alg:cloud} is not a contribution to the results of this paper. The algorithm has been implemented, tested, optimised and deployed on clusters, grids and clouds. It has been -- and is being -- used to attack  {\CSP} instances requiring  an estimated  200 CPU years for exact solution \cite{Distler2010}.

%%%----------------------------------------------------------------------------
\subsubsection{Distributed  {\CSP}s  Using Minion}
\label{distcsp}

In common with Integer Programming problems,  {\CSP}s distribute naturally
across multiple compute nodes \cite{Etzioni2001}. Significant research has been
invested in the distribution of {\CSP}s across multiple
computers~\cite{collin_feasibility_1991,yokoo_distributed_1998,michel_parallelizing_2007}.
In particular the area of balancing the load among the nodes is an area of
active research~\cite{rolf_load-balancing_2008}.

Instead of the more sophisticated approaches, we choose a simple technique that
does not impose any constraints on the problem to be solved and is targeted
towards very large problems. Our algorithm closely follows
Algorithm~\ref{alg:cloud} -- we run Minion with a time out and when this time
out is reached, we split the remaining search space into two parts. The
subproblems are inserted into a FIFO queue and processed by the computational
nodes, splitting them again if necessary. %Figure~\ref{fig:miniondist} gives
%a high-level overview of our approach.

%\begin{figure}
%\includegraphics[angle=-90,scale=.8]{distri}
%\caption{High-level overview of our {\CSP} distribution approach with Minion}
%\label{fig:miniondist}
%\end{figure}

One of the drawbacks of our approach is that it does not parallelise small
problems well. For $n$ compute nodes, we only achieve full capacity utilisation
after $\log_2{n}$ splits, i.e.\ after $\mathtt{timeout}\times\log_2{n}$ seconds.
We do not consider this to be a limiting factor however because the split
timeout can be adapted dynamically to at first quickly split the problem and
when full utilisation has been achieved increase it. For the large problems we
have focused on when implementing this technique, requiring days or even years of CPU time, this is not a limiting factor.

The main advantage of our way of distributing problems over other approaches is
that we explicitly keep the split subproblems in files. This means that at any
point we can stop, suspend, resume, move or cancel the computation and lose a
maximum of $\mathtt{timeout}\times n$ seconds of work, much less in practice.
Apart from contributing to the robustness of the overall system, we can also
easily move subproblems that cannot be solved using the available computational
resources, for example because of memory limitations, to nodes with a higher
specification that are not always available to us.

In the absence of global symmetry breaking constraints that can affect different
parts of the search tree, it is easy to subdivide a typical  {\CSP} into several
non-overlapping sub-problems.  Although there is an inherent latency in sending
problem instances to, and receiving solutions from, either a grid or a cloud,
for large enough problems a speedup linear in the number of compute nodes is
achieved. Recent results using a computational grid indicate that a super-linear
speedup can be achieved using Minion, whenever  a root node consistency check
reduces the search tree \cite{Distler2010}. There is no guarantee of this, however,
since root consistency checks are heuristics that will  at times provide no
benefit for the extra work involved. 

Cloud computing is becoming an important computational paradigm, and the Minion
developers have produced robust, fault-tolerant, methods for distributing Minion
instances across different underlying architectures, including clouds. By
leveraging existing technologies, in particular the Condor distributed computing
framework~\cite{beowulfbook-condor}, we can distribute problems across hundreds
of CPUs and combine cluster, grid and cloud architectures for web-scale
computing. This enables us to tackle problems which have previously been thought
to be unsolvable because of the amount of computation required to find a
solution.

%%%-------------------------
\subsection{Distributed Closest String}
\label{distclosest}

\IncMargin{1em}
\begin{algorithm}[ht]
  \SetKwBlock{Parallel}{do in parallel}{end}
      \SetKwFunction{Solved}{Solved?}\SetKwFunction{RuledOut}{RuledOut?}
  \SetKwInOut{Input}{Input}\SetKwInOut{Output}{Output}

  \Input{$\Upsilon(S,d_{min}, HD(S))$, $T_{max}$ and $K$ }
  \Output{A closest string to $S$ with its maximum Hamming distance  to  $S$}
  \BlankLine
 % \emph{Seek solution for $\Upsilon^0(S,d_{min},HD(S))$ without using the cloud}\;
  \For{$0 \leq t \leq T_{max}$} 
           {Run$\Upsilon(S,d_{min},HD(S))$ in Minion \\
           \If{\Solved{$\Upsilon(S,d_{min},HD(S))$}}{Return $CS$ and $d$, and halt all computation\\
           \lElse{ $d_{high} \leftarrow$ the best $d$ found so far\\
            $\Upsilon(S,d_{min},HD(S)) \leftarrow \Upsilon(S,d_{min},d_{high})$ plus constraints ruling out search already performed}
           }
          }
 %    \emph{Seek solution for small distances, or rule these distances out}\;
    \Parallel{%
    \For{$d_{min}  \leq  d_{low} < d_{high}$} 
         {DistSolve($\Upsilon^*(S,d_{min},d_{low})$, $T_{max}$, $K$)\\
         \If{\Solved{$\Upsilon^*(S,d_{min},d_{low})$}}{Return $CS$ and $d = d_{low}$, and halt all computation\\
         \lElse{
         Update all (sub-)instances with new lower bound $d_{low} + 1$}
         }
         }
%     \emph{Contemporaneously seek solution for decreasing large distances}\;
     \For{$2 \leq k \leq K+1$} 
         { DistSolve($\Upsilon_k(S,d_{min},d_{low})$, $T_{max}$, $K$)\\
         \If{\Solved{$\Upsilon_k(S,d_{min},d_{high})$} with $d_{new} < d_{high}$ }{
         \lIf{$d_{new} = d_{high} - 1$}{
           {Return $CS$ and $d = d_{new}$, and halt all computation}\\
           \lElse{Update all (sub-)instances with upper bound $d_{high} = d_{new}$}
           }
         }
         }
     }
%        \emph{If all distances below a current $d_{high}$ have been ruled out, then return it}\;
         \If{not \Solved{ $\Upsilon_k(S,d_{min},d_{high}) ~\forall ~k$}  $\wedge$ not \Solved{any fixed $d_{low}$ instance}}
              {Return current $d_{high}$ as $d$, and the string found that achieved distance  $d_{high}$ as $CS$ }
           
  \caption{Solve the {\csp} $\Upsilon(S,d_{min},HD(S))$  by distributing search for high and low distances}\label{alg:csp}
\end{algorithm}\DecMargin{1em}

Algorithm \ref{alg:csp} describes our approach to the distributed solution of
closest string problems formalised as Constraint Satisfaction problems. We first
run Minion on the original problem with the PWM ordering heuristic as a single
process. Our empirical evaluation in Section \ref{compare} indicates that nearly
always this process will highly efficiently lower the upper bound for the
problem. Once we have a reasonable upper bound, we start searching for the
optimal distance both above and below. From above, we carry on optimising as
before, but we use the recursive DistSolve algorithm to distribute. From below
we create instances each having a fixed distance, the idea being to exhaustively
rule out any closest strings at these distances. These instances are  run on the
computational nodes at the same time as the optimisation sub-problems. If at any
stage we obtain a candidate closest string at a distance for which all lower
distances have been ruled out, then this is our solution. This can happen both
from above and below.

As mentioned in Section \ref{distcsp}, we expect a super-linear speedup by
performing a root node consistency check for each sub-instance. By keeping track
of the best distance obtained so far during search from above, and of any lower
distances for which no solution has been found, we expect to obtain a further
super-linear speedup in the majority of instances. A large part of the search
tree is pruned by updating all instances (either waiting for input to a
node, or currently being processed by a node) with improved distance bounds as
they become available.

%%------------------------------------------------------------------------------
\subsection{Preliminary Evaluation of Distributed Closest String}
\label{evaldist}

For a first evaluation, we ran the algorithm on 6 random strings of lengths 25,
26, 27, 28, 29 and 30. We chose a time limit of 1 hour to reduce communication
overheads. The problems with strings of length 25, 26 and 27 were solved
to completion within this limit.

The remaining three instances were split after one hour and distributed across
multiple machines. As suggested by Figure~\ref{convergence}, the solutions
converged towards the optimal distance extremely quickly. For only one of the
instances was a better Hamming distance  found in one of the sub-instances. The
remaining sub-instances proved the optimality of the previously found
solution.

These tests demonstrate the practical applicability  of our distributed approach. We have not performed a large-scale evaluation, nor have we obtained evidence for the super-linear speedups associated with bounds updates and an increased number of consistency checks at the root of sub-instances.  Our experience with the distributed solution of other classes of {\CSP} suggests that our system will scale seamlessly to grids or clouds containing an essentially unlimited number of compute nodes: there is no communication across nodes, a node failure can be recovered from with no extra search needed (the search tree already explored is reported whenever search is interrupted for any reason), and the order in which sub-instances are solved can be tuned.

%%------------------------------------------------------------------------------
\subsection{Hybrid Methods}
\label{hybrid}

\IncMargin{1em}
\begin{algorithm}[ht]
      \SetKwFunction{Solved}{Solved?}
  \SetKwInOut{Input}{Input}\SetKwInOut{Output}{Output}

  \Input{$\Upsilon^0(S,d_{min},HD(S))$\\  
    $TOL$,  a limit for the gap between the highest and lowest computed distances }
  \Output{A closest string to $S$ with its maximum Hamming distance  to  $S$ }
  \BlankLine
  \emph{Seek closer distance bounds for  $\Upsilon^0(S,d_{min},HD(S))$  using {\CSP} alone}\;
  \While{$|d_{high} - d_{low}| < TOL$} 
           {Run Algorithm \ref{alg:csp} on $\Upsilon^0(S,d_{min},HD(S))$\\
           Output $d_{low}$ and $d_{high}$ when updated }
     \emph{Once bounds are close enough, send to numeric IP or linear time search}\;
    \If{$ TOL > 1 \wedge |d_{high} - d_{low}| \leq TOL$} 
         {Formulate the remaining problem as an Integer Programming problem\\
         Search for solution using numeric branch and bound \\
         }
     \If{$ |d_{high} - d_{low}| = TOL = 1$} 
      {Formulate the remaining problem as a fixed $d$ instance\\
          Search for solution using linear time methods \\
         }  
\caption{Solve the {\csp} $\Upsilon(S,d_{min},HD(S))$ using hybrid {\CSP} and numerical methods}\label{alg:hybrid}
\end{algorithm}\DecMargin{1em}

 The empirical results obtained so far suggest that {\CSP} formulation with PWM ordering is an effective approach for ruling out high distances: Minion will often find a first solution very quickly, given the search space involved. However, at least for the approach suggested in this paper, {\CSP} formulation requires much more time to provide a certificate that an optimal solution is indeed optimal. As discussed in Section \ref{complexity}, efficient methods have been reported in the literature for when the upper and lower distance bounds are close \cite{Meneses2004}, and for problems where the distance is fixed \cite{Gramm2001}.  In this Section we propose a hybrid approach that aims to take advantage of the best methods available. Algorithm \ref{alg:hybrid} takes a closest string instance and partially solves it using Algorithm \ref{alg:cloud}. If the upper and lower bounds come to within a pre-defined tolerance, then numeric branch and bound methods are used to solve an Integer Programming formulation of the problem not yet solved by Minion. If the distance under consideration  ever becomes fixed, then the linear time methods set out in \cite{Gramm2001} can be applied. 
 
 It should be stressed that these three methods ({\CSP}, IP branch and bound, and linear time) need not be exclusive: once tolerance achieving bounds are found by Minion, the distributed Minion search can continue, and the {\CSP} and numeric methods are then competing to find the first solution. This of course pre-supposes that computational resource is not a problem, but that is why we are using web-scale facilities in the first place.

%------------------------------------------------------------------------------
\section{Discussion}
\label{discussion}

We have performed (to our knowledge) the first evaluation of Constraint Satisfaction as a computational framework for solving closest string problems. We have shown that careful consideration of symbol occurrences can provide search heuristics that give, in general, several orders of magnitude speedup when computing approximate solutions. We have also shown that {\CSP} is less effective when searching for certificates of distance optimality. This result motivated our detailed description of algorithms for web-scale distributed {\CSP} computation, and also our design of hybrid distributed algorithms that can take advantage of the strengths of both numeric and {\CSP} computational techniques. 

We have also performed (to our knowledge) the first analysis  of the computational difficulties involved in the identification of all closest strings for a given input set, irrespective of the computational framework used.
Our results for all closest strings motivate the question of which definition of self-similarity is suitable for the computational biology setting. In terms of information theory the all closest strings problem can not exclude alphabet symbols and still be correct. However, when seeking to quantify the self-similarity of DNA sequences it may be perfectly justifiable to exclude closest strings that can have no symbol in common with the sequences in question at a given point. If this were to be the case, then the computational efficiency of the search for all closest strings would be greatly increased (Figure \ref{searchall}).

We have designed, implemented and deployed a computational methodology for distributed search for closest string solutions. This contribution provides a practically useful means of attacking the NP-complete instances by division into smaller sub-problems. Our system is guaranteed never to perform the same search twice, will recover seamlessly from any unforseen loss of compute nodes, and is extendable to web-scale clouds. 

The limitations of this study are that we have not been able to compare numeric solutions to {\CSP} solutions directly, (nor assess the hybrid numeric-{\CSP} algorithm described in the paper), and that we have not attacked real world problems in a distributed setting, instead solving randomly-generated instances. We have outlined the possibility of super-linear speedups for distributed search, but present no supporting evidence as our distributed implementation has as yet been tested solely for accuracy, scalability and robustness.

%------------------------------------------------------------------------------
\subsection{Future Work}
\label{future}

Possible future avenues of research include

\begin{itemize}
\item Performing full-scale cloud searches for solutions to real-world closest string problems (rather than concentrating on randomly-generated problem instances as for this paper)
\item  The provision of a fully distributed IP branch and bound solver for use in Algorithm \ref{alg:hybrid}
\item  Experimentation with the directed graph {\csp} formulation described by Meneses et al. \cite{Meneses2004} to improve their IP formulation as an alternative {\CSP} model for closest strings
\item Investigating other NP-hard string sequence problems such as closest substring, farthest string and $n$-mismatch -- an obvious candidate is consensus string, which differs from closest string only in that the objective is to minimise the sum of the distances, rather than minimise the maximum individual distance
\item Search for more complicated metrics than Hamming distance that better capture the concepts of difference and similarity for nucleotide sequences -- recent results involving Markov models \cite{Viswanath2009} suggest that judicious choice of metric has profound implications for both the theory and practice of identifying self-similarity amongst sequences. 
\end{itemize}

%------------------------------------------------------------------------------
\subsection{Acknowledgments}
\label{acks}

The authors would like to acknowledge the critical discussions had with Prof.
Ian Gent and Dr. Ian Miguel, and with various attendees at the International
Society for  Computational Biology's Latin America meeting held in Montevideo,
Uruguay in March 2010. Tom Kelsey is supported by UK Engineering and Physical Sciences Research Council (EPSRC)  grants EP/CS23229/1
and EP/H004092/1. Lars Kotthoff is supported by a Scottish Informatics and Computer Science Alliance (SICSA) studentship. The funders
had no role in study design, data collection and analysis, decision to publish,
or preparation of the manuscript.

%------------------------------------------------------------------------------
% Refs:
%
\label{sect:bib}
\bibliographystyle{plain}
\bibliography{K-ANB2010}

%%------------------------------------------------------------------------------
%------------------------------------------------------------------------------
\end{document}